\documentclass{jssc}

\def\textsubscript#1%
{$_{\text{#1}}$}

\def\cdd{\mbox{\boldmath$\cdot$}~}

\input{amssym.def}
\include{graphix}
\usepackage{multirow}

\makeatletter
\def\@oddfoot{\hfill}
\newcount\shumeicount
\def\setshumei#1#2#3{%
  \shumeicount=\count0
  \def\@oddhead{%
    \raise-5pt\hbox to0pt{\vrule width\hsize height 0pt depth 0.4pt\hss}\relax
    \ifnum \shumeicount=\count0
      \raise-7pt\hbox to0pt{\vrule width\hsize height 0pt depth 0.4pt\hss}\relax
      #1
    \else
      \ifodd\count0
        #2
      \else
        #3
       \fi
     \fi
  }%
}
\makeatother
\makeatletter
\def\@oddfoot{\hfill}
\newcount\shujiaocount
\def\setshujiao{%
  \shujiaocount=\count0
  \def\@oddfoot{%
      \ifodd\count0
      \else
      \fi
  }%
}
\makeatother
\def\title#1#2#3#4{{
  \vspace*{0.3cm}
  \begin{flushleft} \Large\bf #1\end{flushleft}
  \vspace*{-0.2cm}
      \begin{flushleft}
      \bf #2
      \end{flushleft}
      \footnotetext{\hspace{-6mm} #3\\ #4}}}

\def\dshm#1#2#3#4
{\setshumei{J Syst Sci Complex (#1) #2:
{\thepage--\pageref{LastPage}}\hfill}
            {\hfill {\small #3}\hfill\hbox to0pt{\hss\thepage}}
            {\hbox to0pt{\thepage\hss}\hfill {\small #4}\hfill
            }
            \setshujiao}
\def\drd#1#2
{{\vskip 1cm\small \begin{flushleft}
 #1 \\
 #2\\
\copyright The Editorial Office of JSSC \&  Springer-Verlag GmbH
Germany 2018
\end{flushleft}}}

\def\hat{\widehat}
\def\tilde{\widetilde}

\def\epsilon{\varepsilon}

\begin{document}

\title{Monocular Depth Estimation via Neural Network with Learnable Algebraic Group and Ring Structures}
{\uppercase{Wang} Qianlei  \cdd \uppercase{Chen} Kexun  \cdd \uppercase{ZHANG} Shaolin \cdd \uppercase{GAO}  Hongli \\ \cdd \uppercase{zhang} Chaoning \cdd \uppercase{Li} Tianrui \cdd \uppercase{Qin} Xiaolin  }
{\uppercase{Wang} Qianlei \cdd \uppercase{ZHANG} Shaolin  \cdd \uppercase{Qin} Xiaolin \textbf{(Corresponding author)} \\
Chengdu Institute of Computer Applications, Chinese Academy of Sciences, Chengdu 610213 China. \\   
School of Computer Science and Technology, University of Chinese Academy of Sciences, Beijing 101408 China. 
Email: wangqianlei36@gmail.com, zhangshaolin23@mails.ucas.ac.cn, qinxl2001@126.com. \\
\uppercase{Chen} Kexun \cdd \uppercase{Li} Tianrui  \\
School of Computing and Artificial Intelligence, Southwest Jiaotong University, Chengdu 611756 China. \\
Email: chenkexun@my.swjtu.edu.cn, trli@swjtu.edu.cn \\
\uppercase{GAO}  Hongli \\
School of Mechanical Engineering, Southwest Jiaotong University, Chengdu 611756 China. \\
Email: hongli\_gao@swjtu.cn \\
\uppercase{zhang} Chaoning \\
School of Computer Science and Engineering, University of Electronic Science and Technology of China, Chengdu 610054 China. Email: chaoningzhang1990@gmail.com 
   } 
{$^*$This work was supported by Sichuan Science and Technology Program (2025JDDQ0008, 2024NSFJQ0035), 2025 Sichuan Major Science and Technology Special Project ``Unveiling the List and Commanding the Lead" (Grant No. 2025ZDZX0014), and the Talents by Sichuan provincial Party Committee Organization Department.}

\drd{DOI: }{Received: x x 20xx}{ / Revised: x x 20xx}


\dshm{20XX}{XX}{A TEMPLATE FOR JOURNAL}{\uppercase{Surname1
Firstname1} $\cdd$ \uppercase{Surname2 Firstname2} $\cdd$
\uppercase{Surname3 Firstname3}}

\Abstract{Monocular depth estimation (MDE) has witnessed remarkable progress driven by Convolutional Neural Networks and transformer-based architectures. However, these approaches typically treat the problem as a generic image-to-image regression on Euclidean grids, thereby overlooking the intrinsic algebraic and geometric structures induced by perspective projection. To address this limitation, we propose LAGRNet, a novel framework that fundamentally grounds MDE in algebraic geometry by explicitly embedding learnable group, ring, and sheaf structures into the deep learning pipeline. Modeling feature maps as sections of a sheaf over an approximated image manifold, our method first establishes a Group-defined Feature Manifold (GFM) parameterized by a learned algebraic group action to enforce projective equivariance and robustness against view changes. To facilitate algebraically consistent cross-scale interactions, we subsequently introduce a Ring Convolution Layer (RCL) that formulates feature fusion as a graded ring homomorphism. Furthermore, to ensure global topological consistency, a Sheaf-based Module (SM) aggregates local depth cues via Čech nerve on the image topology. Extensive zero-shot evaluations across the KITTI, NYU-Depth V2, and ETH3D benchmarks demonstrate that LAGRNet significantly outperforms state-of-the-art methods in both accuracy and generalization capabilities. Code is available at \textit{https://github.com/Casit-ARIS-WQL/LAGRNet}.}      

\Keywords{Monocular depth estimation, algebraic group, algebraic ring, deep learning.}        



\section{Introduction}
Estimating scene depth from a single RGB image is a fundamental yet ill-posed problem in computer vision, serving as a cornerstone for applications ranging from autonomous driving and robotics to augmented reality. Over the past decade, Convolutional Neural Networks (CNNs) have advanced the field by leveraging translation equivariance, while recent Transformer-based architectures have further pushed performance boundaries by capturing long-range dependencies and global context. Despite these successes, most existing approaches model monocular depth estimation primarily as a statistical learning problem on Euclidean grids. This perspective often overlooks the profound algebraic and geometric constraints imposed by the physical process of image formation, particularly the structures induced by perspective projection.

In classical projective geometry, the relationship between a 2D image and the 3D world is governed by algebraic laws. Every pixel corresponds to a ray intersecting the world along an algebraic variety, and the transformation between views adheres to the action of the projective linear group, PGL(3). Furthermore, the interaction of features across different scales can be formalized through graded rings, and the consistency of local geometric information across an image domain is naturally described by sheaf theory and Čech nerve. Existing deep learning methods typically approximate these geometric properties implicitly through data augmentation or learn them as statistical correlations, lacking a principled mechanism to enforce such algebraic constraints explicitly.

We posit that incorporating learnable algebraic structures directly into the network architecture can provide the strong inductive biases necessary for robust and generalizable monocular depth estimation. To this end, we introduce LAGRNet, the first neural network framework to explicitly embed algebraic group, ring, and sheaf structures into the depth estimation pipeline. Our approach moves beyond standard Euclidean regression by casting feature extraction and fusion as operations on algebraically defined manifolds and modules.

Our contributions are threefold: 
\begin{itemize} 
	\item We propose a Group-defined Feature Manifold (GFM), which models latent features as sections on a manifold acted upon by a learned algebraic group approximating PGL(3). By aggregating features over the group orbit, the GFM explicitly enforces projective equivariance, enabling the network to generalize robustly across diverse viewpoints and perspective distortions. \item We introduce a Ring Convolution Layer (RCL) to fundamentally reframe multi-scale feature fusion. By treating the feature space as a graded module over a graded ring, RCL formulates convolution as a ring homomorphism. This ensures that cross-scale interactions obey strict algebraic alignment via a discrete Cauchy product, preserving geometric consistency during upsampling. 
	\item We design a Sheaf-based Module (SM) to enforce global topological coherence. Interpreting the feature map as a sheaf over the image topology, the SM utilizes Graph Convolutional Networks on the Čech nerve to propagate context and minimize cohomological obstructions, thereby correcting local inconsistencies and ensuring a topologically valid depth map. \end{itemize}

Extensive zero-shot evaluations on the KITTI, NYU-Depth V2, and ETH3D datasets demonstrate the efficacy of our approach. LAGRNet achieves state-of-the-art performance, significantly improving upon recent methods in both absolute relative error and accuracy metrics, validation that embedding algebraic geometry principles yields a more robust and geometrically consistent solution for monocular depth estimation.

\begin{figure}
	\centering
	\includegraphics[width=1\columnwidth]{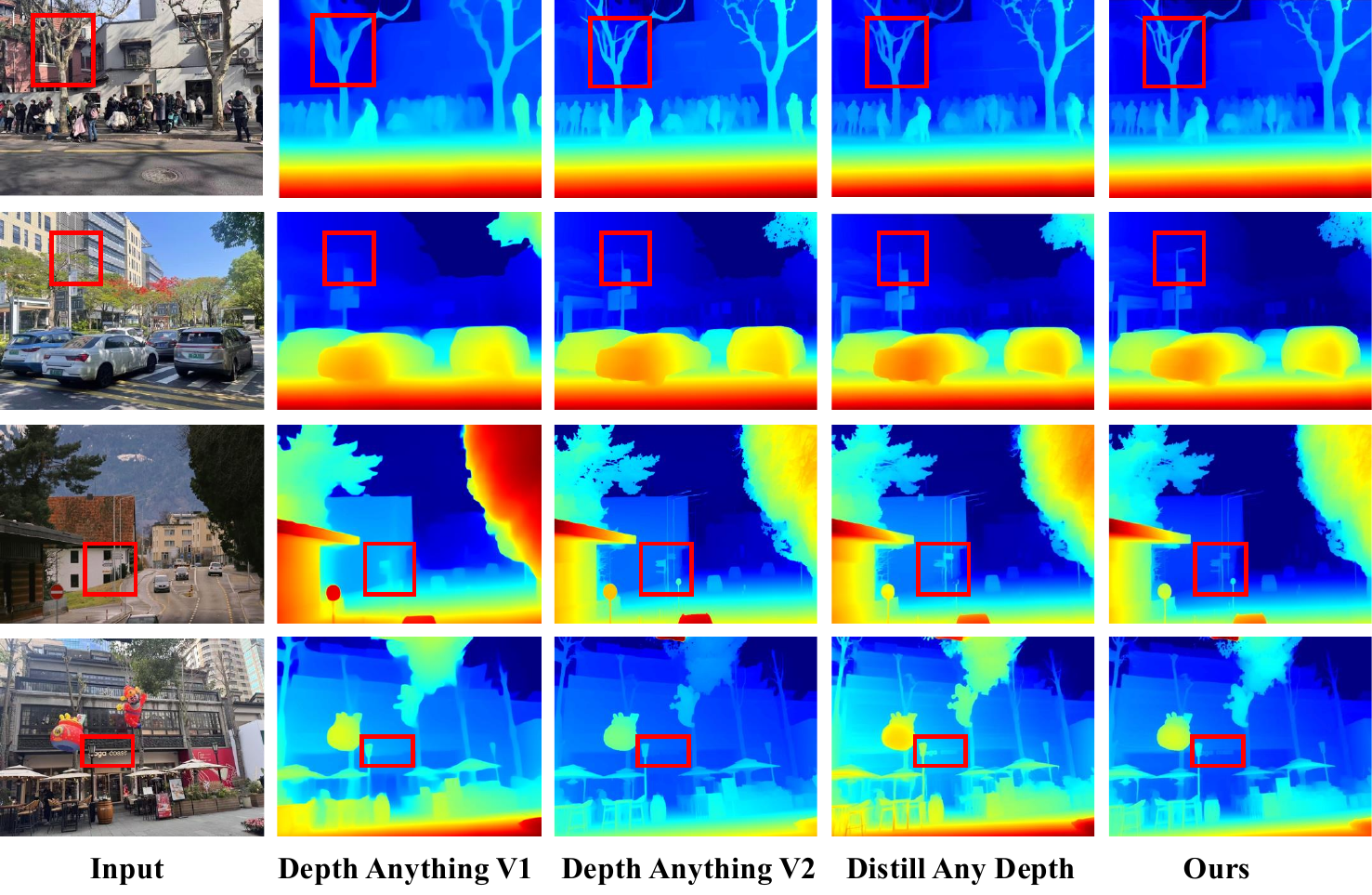 } 
	\caption{Visual results of monocular depth estimation on real-world scenes. In comparison with SOTA methods, the proposed method exhibits more accurate estimation of local details, as indicated by the red rectangles, and delivers improved overall coherence in depth prediction. }
	\label{fig1}
\end{figure}

\section{Related Work}
\subsection{CNN-based Methods}
CNN-based methods advanced MDE advanced through three interconnected directions: enhancing contextual integration, reducing reliance on labeled data, and optimizing data strategies. 
Early approaches \cite{li2015depth, liu2015learning, wang2015towards} focus on contextual refinement, using segmentation to link local details and global context.
Self-supervised methods avoid labeled data, using unlabeled video sequences \cite{zhou2017unsupervised, almalioglu2019ganvo} or stereo pairs \cite{garg2016unsupervised, wang2019unos}.

Key data strategy advancements MDE focus on boosting generalization: DiverseDepth \cite{yin2002diversedepth} enhances affine invariance in depth prediction by leveraging diverse datasets, while Omnidata \cite{eftekhar2021omnidata} introduces a scalable pipeline for generating multi-task vision datasets from 3D scans, providing high-quality depth-relevant training resources to strengthen CNN-based MDE performance.

\subsection{Transformer-based Methods}
Transformers surpass CNNs in modeling local and global depth cues through self-attention mechanisms, with hybrid and pure Transformer designs driving progress.  
Hybrid designs \cite{bhat2021adabins, zhao2022monovit, bae2023deep} integrate CNNs’ strengths in spatial locality with transformers’ prowess in global context modeling, enhancing depth estimation by bridging fine-grained details and high-level scene understanding.

Pure transformer frameworks enhance scalability and generalization. MiDaS \cite{birkl2023midas} pioneered zero-shot cross-dataset transfer via multi-dataset training, while DPT \cite{ranftl2021vision} replaced CNN backbones with ViT for breakthrough dense prediction. Furthermore, DepthAnything v1 and v2 \cite{yang2024depth, yang2024depthv2} improved fine-grained estimation via large-scale labeled and unlabeled training. To refine context and efficiency, TC-Depth \cite{ruhkamp2021attention} uses spatial-temporal attention for multi-frame geometry, and BANet \cite{aich2021bidirectional} reduces ambiguity via bidirectional attention. In parallel, BinsFormer \cite{li2024binsformer} optimizes depth bins through global-local integration, whereas DepthPro \cite{bochkovskii2024depth} employs multi-scale ViT for fast and detail-preserving inference. 

Recent work bridges transformers with generative models and distillation to advance depth estimation, unlocking new capabilities in generalization and knowledge transfer. GenPercept \cite{xu2024matters} repurposes diffusion models for dense perception through one-step finetuning, while Marigold \cite{ke2025marigold} adapts latent diffusion models such as Stable Diffusion to depth and surface normal tasks with minimal modifications, yielding strong zero-shot generalization. Finally, DistillAnyDepth \cite{he2025distill} further enhances transformer-based estimators via cross-context distillation by integrating global and local cues to refine pseudo-label quality.

In contrast, our LAGRNet stands out by explicitly embedding algebraic geometry structures, unlike CNNs and transformers that only implicitly encode geometric geometric cues through heuristics. 
We introduce a Group-defined Feature Manifold with algebraic group actions, enabling view-equivariant features for robust cross-perspective generalization, and a Ring Convolution Layer that models convolutions as graded ring morphisms to capture multi-scale interactions algebraically. By leveraging projective geometry's inherent algebraic structures, LAGRNet moves beyond Euclidean grid-based regression, offering a more principled MDE solution.

\section{Methodology}

\subsection{Overview}
The overall architecture of the proposed LAGRNet, illustrated in Figure 2, fundamentally re-envisions the monocular depth estimation pipeline by embedding algebraic structures into a deep learning framework. Unlike standard approaches that treat feature extraction and fusion as purely statistical operations on Euclidean grids, our method enforces strict geometric equivariance, multi-scale algebraic alignment, and global topological consistency throughout the inference process.

The pipeline begins with a hierarchical encoder based on the Swin Transformer \cite{liu2021swin}, which extracts a pyramid of feature maps from the input image to capture semantic and spatial information across three distinct scales. The highest-resolution level of this hierarchy, the feature stream enters the Group-defined Feature Manifold (GFM). This module models the latent features as sections on a manifold acted upon by a learned projective group. By transforming and aggregating features over the group orbit, GFM explicitly injects rotation and scale equivariance into the semantic bottleneck, rendering the representation robust to perspective distortions inherent in monocular views.

\begin{figure}
	\centering
	\includegraphics[width=1\textwidth]{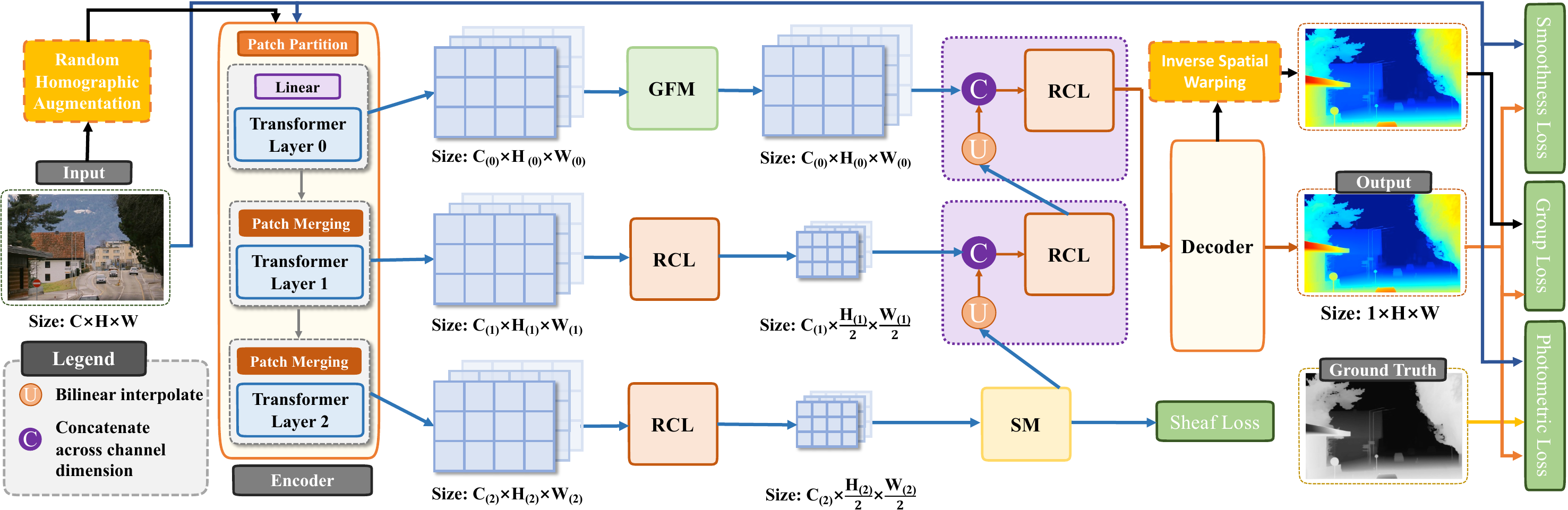} 
	\caption{The overall structure of LAGRNet. We use Swin-T-Base as the backbone, and a vanilla up-sampling network as decoder.}
	\label{fig2}
\end{figure}

To progressively recover spatial resolution while maintaining algebraic consistency, the features are subsequently decoded through a U-shaped structure governed by Ring Convolution Layers (RCL). Rather than relying on independent convolutional operations, RCL treats the multi-scale feature space as a graded module over a graded ring. This design enforces a ring homomorphism constraint via the discrete Cauchy product fusion, ensuring that features from heterogeneous scales are aligned mathematically during the upsampling and concatenation process.

Following the features aligned, the deep features are immediately refined by the Sheaf-based Module (SM) to ensure global topological consistency at the semantic bottleneck. By interpreting the low-resolution feature map as a sheaf over the image topology, the SM utilizes Graph Convolutional Networks on the Čech nerve to propagate global context. This step explicitly minimizes the cohomological obstruction to forming a global section, thereby correcting local inconsistencies before the resolution is restored.

The entire network is optimized end-to-end using a comprehensive objective function that harmonizes visual fidelity with geometric constraints. This includes a photometric loss for synthesis quality, a group consistency loss for geometric invariance, a sheaf energy loss for topological continuity, and an edge-aware smoothness term. Through this algebraically structured design, LAGRNet effectively transforms local visual cues into a globally consistent and geometrically robust depth estimation.

\subsection{Backbone}
To extract robust multi-scale representations, we utilize the Swin Transformer as our feature extractor $\mathcal{B}(\cdot)$. Unlike standard CNNs, the Swin Transformer leverages hierarchical vision transformers with shifted windows. This architecture allows the network to model long-range dependencies through self-attention mechanisms while maintaining computational efficiency and preserving local geometric details, which are critical for dense depth prediction.

Specifically, given an input image $I \in \mathbb{R}^{B \times 3 \times H \times W}$, the backbone first splits the image into non-overlapping patches. These patches are then processed through a series of Swin Transformer blocks containing Window-based Multi-head Self-Attention (W-MSA) and Shifted Window-based Multi-head Self-Attention (SW-MSA). To balance the trade-off between high-level semantic abstraction and low-level spatial resolution, we discard the final stage of the standard Swin backbone and explicitly extract features from the first three hierarchical stages:
\begin{equation}
\left\{\mathcal{F}_{(i)}\right\}_{i=0}^2=\mathcal{B}(I), \quad \text{s.t.} \quad \mathcal{F}_{(i)} \in \mathbb{R}^{B \times C_{(i)} \times \frac{H}{2^{i+2}} \times \frac{W}{2^{i+2}}},
\end{equation}
where $i \in \{0, 1, 2\}$ denotes the scale index. Consequently, we obtain a pyramid of feature maps with strides of $\{4, 8, 16\}$ relative to the input resolution. Among these, the high-resolution feature map $\mathcal{F}_{(0)}$ retains fine-grained textures and object boundaries. To capture intrinsic geometric structures at the finest granularity, $\mathcal{F}_{(0)}$ is directly fed into the GFM to learn geometrically equivariant representations. As the hierarchy deepens, $\mathcal{F}_{(1)}$ captures intermediate structural information, while the deepest map $\mathcal{F}_{(2)}$ encodes rich semantic context and global scene layout. These two coarser scales are processed by the RCL for spectrally aligned multi-scale fusion. The RCL-refined semantic feature $\mathcal{F}_{(2)}$ is forwarded to the SM, where global consistency is enforced through topological diffusion on the Čech nerve. By capping the maximum downsampling factor at 16 rather than the standard 32, we effectively mitigate the excessive loss of spatial fidelity often encountered in deeper network stages, thereby ensuring precise depth recovery even in complex scenes.

\subsection{ Group-defined Feature Manifold}

Monocular depth estimation fundamentally requires geometric equivariance: as the scene undergoes a projective transformation due to viewpoint changes, the estimated depth map must transform commensurately. Standard CNNs, restricted to fixed Euclidean grids, lack the inherent capacity to model such continuous projective distortions. To address this, we introduce the GFM. This module reformulates feature extraction by lifting the representation from the 2D plane to the PGL(3) to explicitly construct the feature orbit.Rather than assuming a static spatial prior, GFM dynamically approximates the continuous group orbit by sampling a discrete set of transformations that span the local geometric variations. We employ a conditional generator to propose candidate elements from the group manifold. 

Let the input global feature be $\mathcal{F}_{(0)} \in \mathbb{R}^{B \times C_0 \times H \times W}$. We generate a set of $K$ matrix representatives $\{\theta_{b,j}\}_{j=1}^K \subset \text{PGL(3)}$ for the $b$-th sample, which serve as the orbit discretization samples. Since $\text{PGL(3)} \cong \text{GL} / \mathbb{R}^*$, any matrix $\mathbf{M}$ is equivalent to $\lambda \mathbf{M}$. To strictly parameterize unique elements and enforce numerical stability, we impose a unit Frobenius norm constraint via a canonical mapping $\Phi$:
\begin{equation}
\theta_{b,j} = \Phi\left( W_2 \operatorname{ReLU} \left(W_1 \mu_b + b_1\right)+b_2 \right), \quad \text{s.t.} \quad \Phi(\mathbf{M}) = \frac{\mathbf{M}}{||\mathbf{M}||_F + \epsilon}
\end{equation}
where $\mu_b = \text{GlobalAvgPool}(\mathcal{F}{(0)})$ denotes the global context vector, $\epsilon$ prevents division singularities, and $W_1, W_2$ are learnable parameters. This parameterization ensures that the generated transformations effectively cover the manifold of planar homographies while maintaining a compact representation.We formalize the group action $g \cdot \mathcal{F}$ as a differentiable coordinate transformation on the projective plane. Let $\Omega = [0, H-1] \times [0, W-1]$ denote the discrete feature grid domain. We bridge Euclidean and projective spaces via two auxiliary mappings defined as the embedding $\iota(u, v) = [u, v, 1]^\top$ and the dehomogenization $\pi(x_1, x_2, x_3) = (x_1/x_3, x_2/x_3)$. The inverse geometric transformation maps the target coordinate $p \in \Omega$ back to the source frame via $\tilde{p}' \sim \theta_{b,j}^{-1} \cdot \iota(p)$. The transformed feature value $\eta_{b,c}^{(j)}(p)$ is computed via differentiable bilinear interpolation:
\begin{equation}
\eta_{b,c}^{(j)}(p) = \sum_{q \in \mathcal{N}(\pi(\tilde{p}'))} \mathcal{F}{(0)b,c}(q) k{\mathrm{bil}}\left(\pi(\tilde{p}') - q\right),
\end{equation}
where $k_{\mathrm{bil}}(\cdot)$ is the bilinear kernel and $\mathcal{N}$ denotes the 4-pixel neighborhood. Finally, we perform an attention-driven integration over the approximated orbit. A view-selection vector $\bm{\alpha}_b \in \Delta^{K-1}$ modulates the transformed features, and the final output is obtained by linear projection:
\begin{equation}
F_{\text{GFM}} = W_{\text{proj}} \left( \sum_{j=1}^{K} \alpha_{b,j} \cdot \eta^{(j)} \right),
\end{equation}
where $W_{\text{proj}}$ is a learnable linear projection. This aggregation effectively marginalizes out the geometric variations along the orbit, yielding a robust equivariant representation.

\begin{figure}
	\centering
	\includegraphics[width=1\columnwidth]{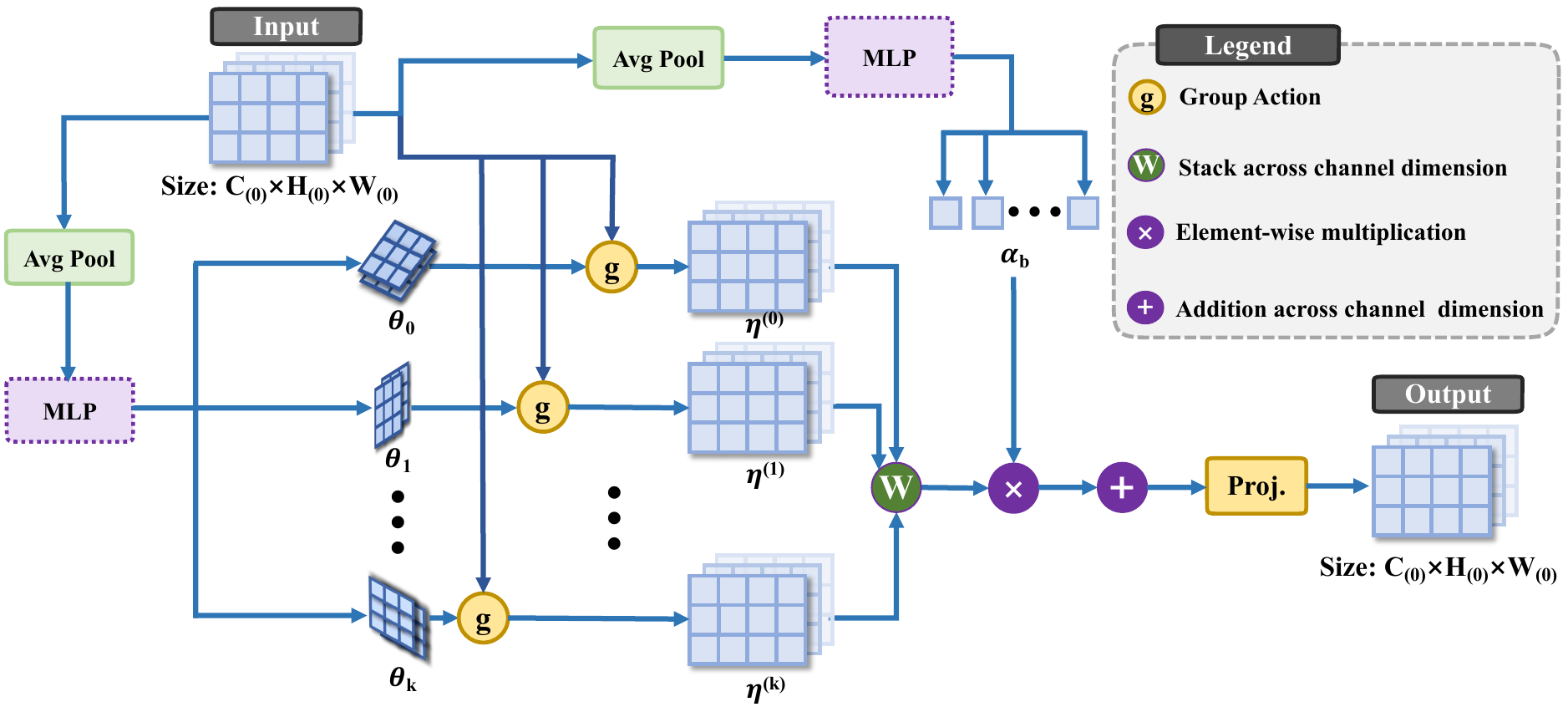 } 
	\caption{Details of Group‑defined Feature Manifold.}
	\label{fig3}
\end{figure}



\begin{definition}
	A feature mapping $\Psi: L^2(X) \to L^2(X)$ is $G$-equivariant if it commutes with the group representation $\rho$:
	\begin{equation}
	\Psi(\rho(g) \mathcal{F}) = \rho(g) \Psi(\mathcal{F}), \quad \forall g \in G, \mathcal{F}\in L^2(X)
	\end{equation}
\end{definition}

\begin{proposition}
	The GFM module is $G$-equivariant, provided that the generator $\mathcal{G}$ satisfies the projective conjugation property where the induced transformations follow $\theta' \sim g \theta g^{-1}$ (up to a scalar), the attention weights $\alpha$ remain geometrically invariant satisfying $\alpha(\rho(g)\mathcal{F}) = \alpha(\mathcal{F})$, and the linear projection $W_{\text{proj}}$ commutes with the spatial warping operator $\rho(g)$.
\end{proposition}

\textit{Proof.}
We proceed with a constructive proof by verifying the commutation relation. First, we explicitly define the GFM mapping as $\Psi(\mathcal{F}) = W_{\text{proj}} \sum_j \alpha_j \eta^{(j)}$, where the warped feature is given by $\eta^{(j)}(\tilde{x}) = [\rho(\theta_j)\mathcal{F}](\tilde{x})$ and $\theta_j \in \mathcal{G}(\mathcal{F})$ is the $j$-th transformation output by the generator for $\mathcal{F}$. Let $\mathcal{F}' = \rho(g)\mathcal{F}$ denote the transformed input, satisfying $\mathcal{F}'(\tilde{x}) = \mathcal{F}(g^{-1}\cdot\tilde{x})$. Under the projective conjugation property, the generated transformation for $\mathcal{F}'$ adapts as $\theta'_j = \lambda (g\theta_j g^{-1})$ for some $\lambda \neq 0$, where $\theta'_j \in \mathcal{G}(\mathcal{F}')$. The inverse transformation is derived via group inverse properties:
$$(\theta'_j)^{-1} = \lambda^{-1} g\theta_j^{-1} g^{-1}.$$

We explicitly compute the warped feature for the transformed input. Leveraging the scale invariance of projective space ($\lambda^{-1}\tilde{x} \sim \tilde{x}$, ensuring the well-definedness of $\text{PGL}(3)$ action), we have:
\begin{equation}
\begin{aligned}
\eta'^{(j)}(\tilde{x}) &= \mathcal{F}'\left( (\theta'_j)^{-1} \cdot \tilde{x} \right) \\
&= \mathcal{F}'\left( g\theta_j^{-1}g^{-1} \cdot \tilde{x} \right) \\
&= \mathcal{F}\left( g^{-1} \cdot (g\theta_j^{-1}g^{-1} \cdot \tilde{x}) \right)  \\
&= \mathcal{F}\left( (g^{-1}g)\cdot \theta_j^{-1}g^{-1} \cdot \tilde{x} \right)  \\
&= \mathcal{F}\left( e\cdot \theta_j^{-1}g^{-1} \cdot \tilde{x} \right)  \\
&= \mathcal{F}\left( \theta_j^{-1} \cdot (g^{-1}\cdot \tilde{x}) \right) \\
&= [\rho(g)\eta^{(j)}](\tilde{x}).
\end{aligned}
\end{equation}

Next, we substitute $\eta'^{(j)}$ into the aggregation formula. Utilizing the geometric invariance of attention weights ($\alpha'_j=\alpha_j$, by the proposition condition), the linearity of the group representation $\rho(g)$, and the commutativity of $W_{\text{proj}}$ with $\rho(g)$, we derive:
\begin{equation}
\begin{aligned}
\Psi(\mathcal{F}') &= W_{\text{proj}} \sum_{j} \alpha_j' \eta'^{(j)} \\
&= W_{\text{proj}} \sum_{j} \alpha_j \rho(g)\eta^{(j)} \\
&= W_{\text{proj}} \rho(g) \sum_{j} \alpha_j \eta^{(j)}  \\
&= \rho(g) W_{\text{proj}} \sum_{j} \alpha_j \eta^{(j)}  \\
&= \rho(g) \Psi(\mathcal{F}).
\end{aligned}
\end{equation}
This confirms the commutation relation, proving that the GFM module is $G$-equivariant. \hfill $\square$

\subsection{Ring Convolution Layer}
Traditional multi-scale fusion approaches typically treat feature maps at different resolutions as independent vectors, often ignoring their intrinsic geometric relationships. To resolve this, we fundamentally reframe multi-scale processing through the lens of graded algebra. We model the multi-scale feature space as a Graded Module $\mathcal{M}=\oplus_{d=0}^\mathcal{D}\mathcal{M}_{d}$ over a Graded Ring $\mathcal{W}=\oplus_{\ell=0}^\mathcal{D}\mathcal{W}_{\ell}$. Physically, the $d$-th homogeneous component $\mathcal{M}_d$ corresponds to the feature map subspace at a downsampling factor of $2^d$, where the algebraic degree $d$ strictly correlates with the spatial resolution scale.

Instead of learning independent convolutions for each scale, we strictly define the feature fusion as a graded ring action. To enforce parameter efficiency and structural coupling, we introduce a shared generator kernel $W_{RCL} \in \mathbb{R}^{C_{out} \times C_{in} \times k \times k}$. Here, $C_{in}$ and $C_{out}$ denote the unified input and output channel dimensions of RCL, and $k$ represents the spatial kernel size. The Graded Ring $\mathcal{W}$ is realized as the set of derived kernels $\{W_{\ell}\}_{\ell=0}^\mathcal{D}$ extracted from this unified parameter pool. The specific degree component $W_{\ell}$, representing convolution kernels operating at algebraic degree $\ell$, is derived via $W_{\ell}=W_{RCL}\odot M_{\ell}$. Crucially, to satisfy the static definition of a ring, we construct the masks $M_{\ell}$ by partitioning the output channel dimension into $\mathcal{D}+1$ disjoint groups where $\mathcal{D}$ is the maximum degree. Defining a channel group size $s = \lfloor C_{out} / (\mathcal{D}+1) \rfloor$, the binary mask $M_{\ell}$ for the $c$-th output channel is mathematically defined as:
\begin{equation}
M_{\ell, c} = 
\begin{cases} 
1, & \text{if } \quad \ell \cdot s \le c < (\ell+1) \cdot s \\
0, & \text{otherwise}
\end{cases}.
\end{equation}
This formulation acts as a rigorous band-pass selector, ensuring that each algebraic degree $\ell$ is mapped to a unique and orthogonal subspace of the output tensor.

The core operation is governed by the graded multiplication rule, strictly following the discrete Cauchy product. Unlike standard fusion which only considers peer-to-peer interactions, our approach enables dense cross-scale interactions governed by algebraic conservation. For a target output degree $d$, the network aggregates all valid pairs of kernels and input features whose degrees sum to $d$. Intuitively, this means a feature at scale $d$ is synthesized not only from the input at the same scale ($\ell=0$), but also from finer inputs at scale $d-\ell$ processed by kernels with adaptive downsampling ($\ell>0$). The graded ring action for the $d$-th component is formulated as:
\begin{equation}
\operatorname{RCL}_d(\mathcal{F}) = \sigma \left( \operatorname{BN} \left( \sum_{\ell=0}^d \left( U (W_{RCL} \odot M_{\ell}) * \mathcal{P}_{(d-\ell)}(\mathcal{F}) \right) \right) \right).
\end{equation}
Here, the summation represents the discrete Cauchy product, strictly satisfying the grading condition $\ell + (d-\ell) = d$. The term $\mathcal{P}_{d-\ell}(\mathcal{F})$ denotes the input feature projection to the finer resolution subspace at degree $d-\ell$, the new feature map represents as $\mathcal{F}_{(d-\ell)}^{'}$. The operator $*$ denotes the convolution operation, and the operator $U$ represents bilinear interpolation which physically realizing the spatial transformation from the finer scale $d-\ell$ to the target scale $d$. Finally, the output components from different degrees $\{ \operatorname{RCL}_d(\mathcal{F}) \}_{d=0}^\mathcal{D}$ are upsampled and fused via learnable weights to restore the full resolution.

 \begin{figure}
	\centering
	\includegraphics[width=1\textwidth]{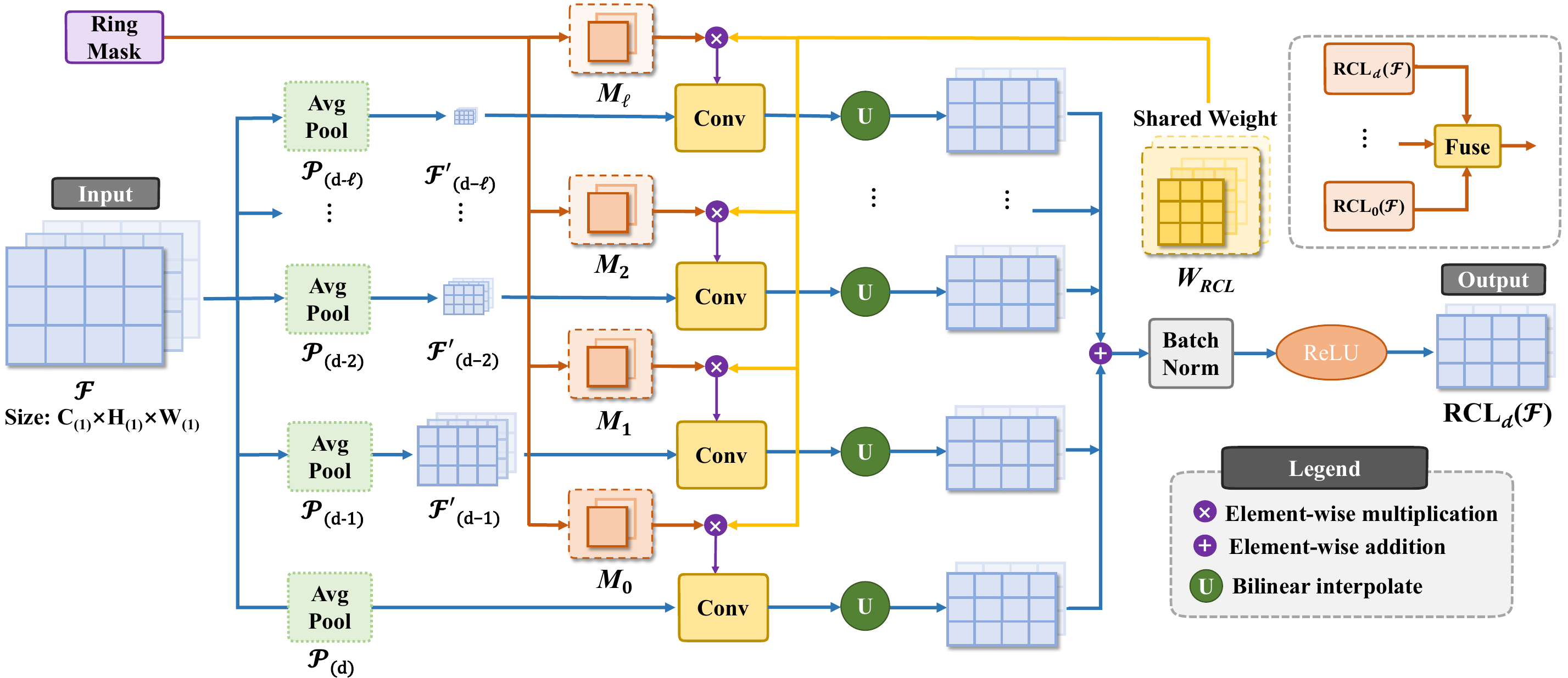} 
	\caption{Details of Ring Convolution Layer.}
	\label{fig4}
\end{figure}

We provide a theoretical analysis demonstrating that the strict degree-binding mechanism of the Graded Ring structure is structurally necessary to satisfy Spectral Moment Matching requirements, thereby resolving the phase alignment error inherent in standard fusion.

\begin{proposition}
	Identity-based fusion is restricted to point-wise operations, yielding a first-order spectral alignment error $\mathcal{O}(|\omega|)$. The discrete Cauchy product, by strictly binding the scale transition $\ell$ to a spatially-extended operator subspace $\mathcal{W}_{\ell}$, provides the necessary geometric capacity to satisfy the First-Order Moment Condition, reducing the error to $\mathcal{O}(|\omega|^2)$.
\end{proposition}

\textit{Proof.} We analyze the spectral response in the frequency domain $\omega \in [-\pi, \pi]^2$. Let $n \in \mathbb{Z}^2$ denote the spatial coordinates. By extracting the linear component of Formula (9), the fusion process simplifies to the Discrete Cauchy product fusion structure $\sum_{\ell=0}^{d} (w_\ell * f_{d-\ell})[n]$, where $w_\ell \in \mathcal{W}_\ell$ represents the effective kernel at scale degree $\ell$, and $f_{d-\ell}$ denotes the input signal component at the source degree $d-\ell$.

We assume the scale-induced misalignment is locally approximable by a translation vector $\Delta_\ell$. Let $\mathcal{T}_{\Delta_\ell}$ be the ideal transport operator corresponding to this misalignment, characterized by the spectral response $e^{-i \langle \omega, \Delta_\ell \rangle}$. Identity-based fusion implicitly restricts the interaction kernel to the identity $\mathbb{I}$, which has a trivial support. Its spectrum is constant, $\hat{w}_{id}(\omega) \equiv 1$. Comparing this to the Taylor expansion of the target phase shift $e^{-i \langle \omega, \Delta_\ell \rangle} = 1 - i \langle \omega, \Delta_\ell \rangle + \mathcal{O}(|\omega|^2)$, the approximation error is:
\begin{equation}
\mathcal{E}_{id}(\omega) = | 1 - (1 - i \langle \omega, \Delta_\ell \rangle + \dots) | = |\langle \omega, \Delta_\ell \rangle| + \mathcal{O}(|\omega|^2) \sim \mathcal{O}(|\omega|).
\end{equation}
This linear phase error confirms that degree-0 algebra ($\mathcal{W}_0$) lacks the geometric degrees of freedom to model spatial shifts. The necessity of the Graded ring structure arises from the physics of scaling, the misalignment magnitude $|\Delta_\ell|$ is a function of the degree $\ell$. To correct this, we utilize the spectral capacity of $w_\ell \in \mathcal{W}_\ell$. Expanding its DFT spectrum into moments:
\begin{equation}
\hat{w}_\ell(\omega) = \textbf{M}_0 - i \langle \omega, \textbf{M}_1 \rangle + \mathcal{O}(|\omega|^2),
\end{equation}
where $\textbf{M}_0 = \sum_n w_\ell[n]$ and $\textbf{M}_1 = \sum_n n \cdot w_\ell[n]$ denote the zeroth and first spectral moments. Minimizing the spectral error requires matching the first two terms of the target expansion:
\begin{equation}
\textbf{M}_0 = 1 \quad \text{and} \quad \textbf{M}_1 = \Delta_\ell.
\end{equation}The solvability of this system depends on the operator's capacity. Crucially, a Identity-based fusion would imply a shared kernel $w$, creating an overdetermined system where a single moment vector $\textbf{M}_1(w)$ must simultaneously satisfy $\textbf{M}_1 = \Delta_\ell$ for varying $\ell$. This is algebraically unsolvable. In contrast, the Graded Ring structure decouples these constraints. It assigns a specific kernel $w_\ell$ to each transition order, providing the independent coefficients necessary to satisfy the vector constraint $\textbf{M}_1(w_\ell) = \Delta_\ell$ for each $\ell$ uniquely. Consequently, there exists an optimal kernel $w_\ell^*$ in the graded component $\mathcal{W}_\ell$ such that the error is suppressed to the second order:
\begin{equation}
\mathcal{E}_{Cauchy}(\omega) = | \hat{w}_\ell^*(\omega) - e^{-i \langle \omega, \Delta_\ell \rangle} | = \mathcal{O}(|\omega|^2).
\end{equation}
Thus, the discrete Cauchy product acts as the selection mechanism that enforces this degree-binding, guaranteeing that the operator assigned to a scale transition possesses the sufficient geometric capacity to solve the Moment Matching equations. \hfill $\square$

\subsection{Sheaf-based Module}
Standard convolutional operators are inherently local, rendering depth estimation susceptible to inconsistencies across large receptive fields. To transcend this limitation, we propose the SM, which draws theoretical inspiration from sheaf theory. Rather than treating the feature map solely as a grid of pixels, we interpret it as a collection of local sections defined over the image topology. While strictly defined on continuous spaces, we adapt the concept of a sheaf to the discrete image domain. We lift features onto a set of overlapping patches and propagate information via the Čech nerve, explicitly modeling the gluing data, which represents the compatibility constraints required to stitch local predictions into a globally consistent depth map.

Formally, let $X$ denote the discrete image domain. We construct a cover $\mathcal{U}=\{U_{i}\}_{i=1}^{N}$ of $X$ using overlapping patches generated by sliding windows. Although these patches are finite sets of pixels, they serve as a combinatorial approximation of an open cover in our discrete topology. The extraction of local features acts analogously to a restriction map. We implement this by projecting features onto the cover via adaptive pooling, initializing a Čech 0-cochain representation $V_{b}\in\mathbb{R}^{N\times C}$. To capture the topological connectivity of this cover, we construct the Čech nerve $\mathcal{N}(\mathcal{U})$. Theoretically, this is a simplicial complex where vertices represent patches and edges represent intersections. For our graph-based propagation, we utilize the 1-skeleton of the nerve, effectively flattening the complex into a graph structure. The edge set is encoded by a binary adjacency matrix $A\in\{0,1\}^{N\times N}$, where an edge $(i, j)$ exists if and only if the patches physically overlap $U_{i}\cap U_{j}\ne\emptyset$. In implementation, since the patches are arranged on a fixed 2D lattice, this overlap condition is equivalent to spatial adjacency. Thus, we construct $A$ based on the connected spatial neighborhood of the patch grid, creating a sparse operator that represents the interaction topology.

Once lifted to the topological graph, we harmonize the local representations using two Graph Convolutional Networks (GCN) layers. This approximates a diffusion process along the nerve, allowing overlapping patches to exchange context and converge towards a consensus. The propagation is defined as:
\begin{equation}
h_{b}^{(1)}=\sigma(\hat{A}V_{b}w_{1}), \quad h_{b}^{(2)}=\sigma(\hat{A}h_{b}^{(1)}w_{2}),
\end{equation}
where $\hat{A}$ denotes the renormalized adjacency matrix, and $w_{\{1,2\}}$ represent the learnable weight matrices responsible for feature projection. $\sigma(\cdot)$ denotes a non-linear activation function. It is crucial to note that this operation is fundamentally distinct from an identity mapping. Since the initial local features $V_b$ inherently contain discrepancies across overlapping regions due to limited receptive fields, an identity mapping would preserve these inconsistencies. Instead, the GCN layers perform feature smoothing, actively propagating information to reduce high-frequency variations on the graph.

The ultimate goal is to ensure that refined local features can be glued to form a valid global section. In sheaf theory, this requires local sections to agree on intersections. We enforce this condition by minimizing the sheaf energy, defined as the Graph Laplacian quadratic form:
\begin{equation}
\mathcal{L}_{sheaf}=\frac{1}{2BN}\sum_{b}\sum_{i,j}A_{ij}||h_{b,i}-h_{b,j}||_{2}^{2}=\frac{1}{BN}\operatorname{Tr}(H_{b}^{\top}LH_{b}).
\end{equation}
Here, $B$ denotes the batch size. $H_b \in \mathbb{R}^{N \times C}$ represents the global feature matrix constructed by stacking the local feature vectors $\{h_{b,i}\}_{i=1}^N$, and $H_b^\top$ denotes its transpose. $L=D-A$ is the combinatorial Graph Laplacian, where $D$ is the diagonal degree matrix defined by $D_{ii}=\sum_{j}A_{ij}$. Physically, this term measures the spectral roughness of the signal. Minimizing $\mathcal{L}_{sheaf}$ explicitly penalizes feature discrepancies between overlapping patches ($A_{ij}=1$). This creates a strong gradient signal that prevents the network from learning a trivial identity function, mathematically forcing the local predictions to converge to a harmonic state where overlapping regions satisfy the consistency conditions required for global coherence.

\begin{figure}
	\centering
	\includegraphics[width=1\textwidth]{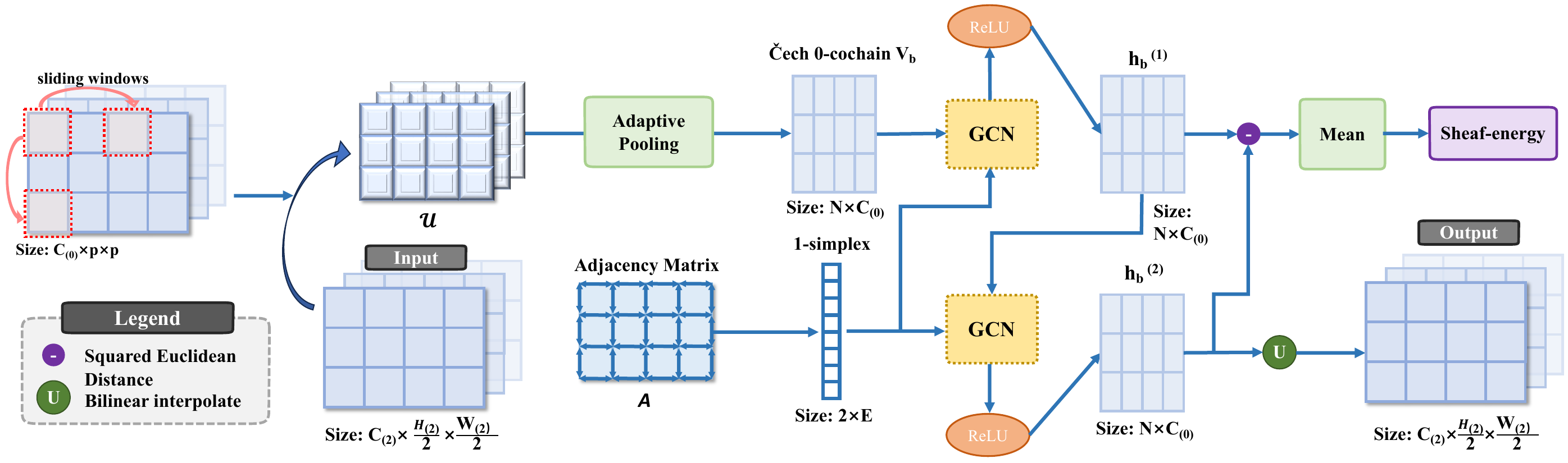} 
	\caption{Details of Sheaf-based Module.}
	\label{fig5}
\end{figure}

\subsection{Objective Loss}
To ensure photometric consistency, geometric invariance, and spatial smoothness, the total network loss $\mathcal{L}$ is formulated as a weighted combination of four terms:
\begin{equation}
\mathcal{L}=\lambda_{\text{pho}} \mathcal{L}_{\text{pho}} + \lambda_{\text{grp}} \mathcal{L}_{\text{grp}} + \lambda_{\text{sheaf}} \mathcal{L}_{\text{sheaf}} + \lambda_{\text{sm}} \mathcal{L}_{\text{sm}},
\end{equation}
where $\lambda_{\{\cdot\}}$ are hyperparameters balancing the contribution of each term. We supervise the network by enforcing photometric consistency between the reference image $I$ and the source images $I_s \in \mathcal{S}$ warped into the reference view. The loss is computed as the L1 difference over valid pixels:
\begin{equation}
\mathcal{L}_{\text{pho}} = \frac{1}{|\mathcal{S}|} \sum_{s \in \mathcal{S}} \left\| I - \mathcal{W}_{s \rightarrow r}(I_s) \right\|_{1, \Omega},
\end{equation}
where $\mathcal{W}_{s \rightarrow r}(I_s)$ denotes the synthesis of the reference view from the source view $s$ using the predicted depth and relative pose. $\|\cdot\|_{1, \Omega}$ denotes the L1 norm computed only over the valid region $\Omega$.

To enforce global consistency on the topological nerve, we utilize the sheaf-energy loss $\mathcal{L}_{\text{sheaf}}$ defined in Formula (13). This term explicitly minimizes the spectral energy of the feature graph, penalizing cohomological obstructions across overlapping patches. To enforce equivariance under geometric transformations, we impose a consistency constraint. Let ${D}_{pred}$ be the depth predicted from image $I$, and ${D}_{pred}'$ be the depth predicted from a transformed image $I' = g(I)$, where $g$ represents a random homographic augmentation. The group-consistency loss aligns these predictions in the canonical coordinate system:
\begin{equation}
\mathcal{L}_{\text{grp}} = \left\| {D}_{pred} - \mathcal{W}_{g^{-1}}({D}_{pred}') \right\|_1,
\end{equation}
where $\mathcal{W}_{g^{-1}}$ denotes the inverse spatial warping operator associated with the group action $g$. This ensures that the depth predictions are equivariance with the geometric transformations of the input. Finally, we apply an edge-aware smoothness loss $\mathcal{L}_{\text{sm}}$ to encourage continuity in textureless regions while preserving sharp discontinuities at object boundaries:
\begin{equation}
\mathcal{L}_{\text{sm}} = \sum_{p} \left( |\partial_x {D}_{pred}(p)| e^{-\gamma |\partial_x I(p)|} + |\partial_y {D}_{pred}(p)| e^{-\gamma |\partial_y I(p)|} \right).
\end{equation}
Here, $p$ indexes the spatial pixel coordinates. ${D}_{pred}(p)$ and $I(p)$ denote the predicted depth value and the corresponding image intensity at pixel $p$, respectively. The operators $\partial_x$ and $\partial_y$ represent the gradients in the horizontal and vertical directions. Finally, $\gamma$ is a scalar coefficient that governs the edge sensitivity, ensuring that the smoothing penalty is attenuated in regions with high image gradients.

\section{Experiments }
\subsection{Datasets and Implementation Details}
We follow the datasets usage of Depth Anything V1 during training,  and perform zero-shot evaluations on three widely-used benchmark datasets for monocular depth estimation: KITTI, NYU-Depth V2, and ETH3D. KITTI contains street-level outdoor scenes with sparse ground-truth depth maps acquired using a LiDAR sensor, and it serves as a challenging real-world benchmark for autonomous driving scenarios. NYU-Depth V2 provides indoor RGB-D images captured using a Microsoft Kinect camera, featuring complex indoor layouts and rich scene diversity, making it ideal for indoor scene understanding. ETH3D includes both indoor and outdoor scenes, providing highly accurate ground-truth depth maps through laser scanners, thus enabling detailed evaluation of model accuracy and robustness. We use the absolute relative error (AbsRel), the accuracy under threshold $(\delta_i<1.25^i, i=1, 2, 3)$, the root mean squared error (RMSE), the root mean squared error in log space ($\text{RMSE}_{lg}$), and the squared relative error (SqRel) as performance metrics.

We implement LAGRNet using PyTorch and train the model using the AdamW optimizer with a learning rate initialized at $1e-4$ ,  and the polynomial decaying method with the power of 0.9. The model is trained on 4 Nvidia H20 GPUs with a batch size of 128 for 480K iterations. During training, input images are resized and augmented with random cropping, flipping, and color jittering. Hyperparameters for loss weights are empirically set as $\lambda_{\mathrm{pho}}=1.0, \lambda_{\mathrm{grp}}=0.5, \lambda_{\mathrm{sheaf}}=0.1$, and $\lambda_{\mathrm{sm}}=0.01$.

\begin{table}
	\begin{center}
		{\small
			\centerline{\small {\bf Table 1}~~Cross-dataset monocular depth estimation benchmark ($\downarrow$ smaller is better, $\uparrow$ larger is better).}\vskip 1mm
			\label{tab:cross_dataset}
			
			\renewcommand{\arraystretch}{1.05} 
			
			\begin{tabular*}{\textwidth}{@{\extracolsep{\fill}}lcccccccc}
				\toprule
				& \multicolumn{2}{c}{NYU-Depth V2} & \multicolumn{2}{c}{KITTI} & \multicolumn{2}{c}{ETH3D} & \multicolumn{2}{c}{Means} \\
				\cmidrule(r){2-3}\cmidrule(r){4-5}\cmidrule(r){6-7}\cmidrule(r){8-9}
				
				Method & AbsRel$\!\downarrow$ & $\delta_1\!\uparrow$ & AbsRel$\!\downarrow$ & $\delta_1\!\uparrow$ & AbsRel$\!\downarrow$ & $\delta_1\!\uparrow$ & AbsRel$\!\downarrow$ & $\delta_1\!\uparrow$ \\[2pt]
				\midrule
				DiverseDepth~\cite{yin2002diversedepth}     & 0.117 & 0.875 & 0.190 & 0.704 & 0.228 & 0.694 & 0.178 & 0.758 \\
				MiDaS~\cite{ranftl2020towards}              & 0.111 & 0.885 & 0.236 & 0.630 & 0.184 & 0.752 & 0.177 & 0.756 \\
				Omnidata~\cite{eftekhar2021omnidata}        & 0.074 & 0.945 & 0.149 & 0.835 & 0.166 & 0.778 & 0.130 & 0.853 \\
				HDN~\cite{zhang2022hierarchical}            & 0.069 & 0.948 & 0.115 & 0.867 & 0.121 & 0.833 & 0.102 & 0.883 \\
				DPT~\cite{ranftl2021vision}                 & 0.098 & 0.903 & 0.100 & 0.901 & 0.078 & 0.946 & 0.092 & 0.917 \\
				GenPercept~\cite{xu2024matters}             & 0.058 & 0.969 & 0.080 & 0.934 & 0.096 & 0.959 & 0.078 & 0.954 \\
				Marigold~\cite{ke2025marigold}              & 0.055 & 0.961 & 0.099 & 0.916 & 0.065 & 0.960 & 0.073 & 0.946 \\
				DepthAnything v1~\cite{yang2024depth}       & \underline{0.043} & \underline{0.981} & 0.076 & 0.947 & 0.127 & 0.882 & 0.082 & 0.937 \\
				DepthAnything v2~\cite{yang2024depthv2}     & 0.045 & 0.979 & 0.074 & 0.946 & 0.131 & 0.865 & 0.083 & 0.930 \\
				DistillAnyDepth \cite{he2025distill}        & \underline{0.043} & \underline{0.981} & \underline{0.070} & \underline{0.949} & \underline{0.054} & \underline{0.981} & \underline{0.056} & \underline{0.970} \\[2pt] 
				\midrule
				Ours & \textbf{0.041} & \textbf{0.987} & \textbf{0.051} & \textbf{0.982} & \textbf{0.048} & \textbf{0.989} & \textbf{0.047} & \textbf{0.986} \\
				\bottomrule
			\end{tabular*}
		}
	\end{center}
\end{table}

\subsection{Comparison Results}
Our method demonstrates superior performance across multiple datasets compared to ten recent SOTA methods. Specifically, we achieve the best results both in terms of AbsRel and $\delta_1$. For NYU-Depth V2, AbsRel decreases from 0.043 to 0.041, representing a 4.7\% relative improvement, while accuracy $\delta_1$ increases from 0.981 to 0.987. In the more challenging cross-domain scenario of KITTI, AbsRel significantly improves from 0.070 to 0.051, accompanied by a notable rise in $\delta_1$ from 0.949 to 0.982 , indicating enhanced generalization capability to street scenes. For ETH3D, our approach reduces AbsRel from 0.054 to 0.048, and $\delta_1$ increases from 0.981 to 0.989, demonstrating robustness even in tightly structured outdoor environments. Meanwhile, the mean accuracy $\delta_1$ reaches 0.986, surpassing the runner-up by 1.6 points and exceeding earlier baselines such as DiverseDepth by more than 22 points.

\begin{figure}
	\centering
	\includegraphics[width=1\textwidth]{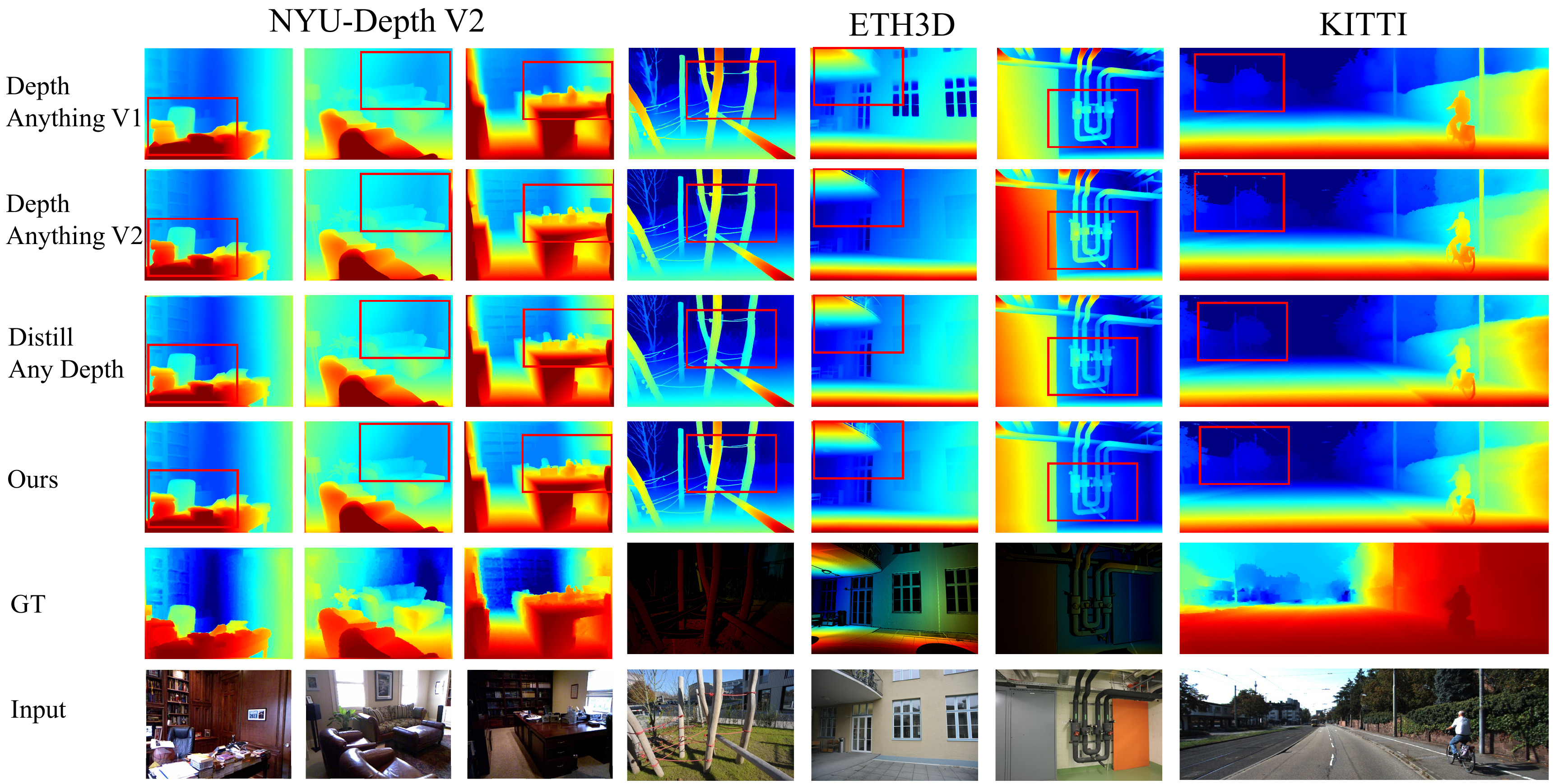} 
	\caption{Qualitative comparison with state-of-the-art methods on three different datasets. The red rectangles highlight the comparison regions. Our method not only predicts more accurate depth on indoor and outdoor scenes but also also recovers better details.}
	\label{fig6}
\end{figure}

Figure \ref{fig6} presents a visual comparison between LAGRNet and current state-of-the-art methods across three diverse benchmarks. As highlighted by the red rectangle, standard methods often suffer from over-smoothing artifacts, leading to blurred boundaries and loss of fine-grained structures. In contrast, LAGRNet exhibits superior geometric fidelity. Specifically, in the ETH3D examples, our method accurately reconstructs intricate topologies such as thin tree trunks and entangled pipes. While competitors tend to fracture or dissolve these high-frequency details due to spectral attenuation. In the indoor scenes of NYU-Depth V2, LAGRNet produces crisp edges around furniture and distinct separation between foreground objects and the background. On the KITTI benchmark, our method maintains consistent depth gradients on road surfaces and distinct object boundaries even in distant regions, demonstrating the generalization capability. Overall, our visual results are perceptually closer to the ground truth, confirming that embedding explicit algebraic structures enables the network to recover precise 3D geometry that pure regression-based methods often miss.

\subsection{Ablation Study}
In this section, we perform a rigorous ablation analysis to validate the efficacy of the proposed algebraic framework. Our investigation is twofold as we first dissect the architecture to isolate the individual contributions of the modules grounded in group, ring, and sheaf theories, proving that each component is essential for geometric robustness. Second, we examine the sensitivity of the system to its intrinsic algebraic hyperparameters including the orbital sampling size and graded ring depth to determine the optimal configuration between performance and computational efficiency. All experiments in this section are conducted on the combined dataset to ensure statistical reliability.
\subsubsection{Impact of Individual Components}
To systematically quantify the contribution of each algebraic module within LAGRNet, we conducted a comprehensive ablation study. We established a minimalistic baseline by connecting the Swin Transformer encoder directly to a standard upsampling decoder which is devoid of any algebraic structures. As presented in Table 2, the results reveal that embedding algebraic priors is fundamental to the performance of the model. The complete LAGRNet achieves a mean AbsRel of 0.047 and $\delta_1$ of 0.986, which represents a substantial 75\% error reduction and a 40\% accuracy gain over the baseline.

\begin{table}
	\begin{center}
		{\small
			\centerline{\small {\bf Table 2}~~Ablation results with different components.}\vskip 1mm
			\label{tab:ablation_components}
			
			\renewcommand{\arraystretch}{1.05}
			
			\begin{tabular*}{\textwidth}{@{\extracolsep{\fill}}lccccccc}
				\toprule
				& \multicolumn{4}{c}{Error Metrics $\downarrow$} & \multicolumn{3}{c}{Accuracy Metrics $\uparrow$} \\
				\cmidrule(r){2-5} \cmidrule(l){6-8}
				
				Method & AbsRel & SqRel & RMSE & RMSE$_{lg}$ & $\delta_1$ & $\delta_2$ & $\delta_3$ \\
				\midrule
				
				Baseline & 0.186 & 0.152 & 0.654 & 0.245 & 0.702 & 0.885 & 0.954 \\
				LAGRNet w/o GFM & 0.066 & 0.035 & 0.268 & 0.105 & 0.956 & 0.988 & 0.996 \\
				LAGRNet w/o RCL & 0.061 & 0.031 & 0.245 & 0.098 & 0.964 & 0.991 & 0.997 \\
				LAGRNet w/o SM & 0.056 & 0.025 & 0.212 & 0.086 & 0.968 & 0.993 & 0.998 \\
				LAGRNet w/o $\mathcal{L}_{\text{sheaf}}$ & 0.054 & 0.022 & 0.198 & 0.082 & 0.971 & 0.994 & 0.998 \\
				\textbf{LAGRNet (Ours)} & \textbf{0.047} & \textbf{0.015} & \textbf{0.182} & \textbf{0.075} & \textbf{0.986} & \textbf{0.996} & \textbf{0.999} \\
				\bottomrule
			\end{tabular*}
			
			\vskip 1mm
			\parbox{\textwidth}{\scriptsize \textit{Note:} Error metrics ($\downarrow$): AbsRel, SqRel, RMSE, and RMSE$_{lg}$ (log). Accuracy metrics ($\uparrow$): $\delta_i < 1.25^i$. Baseline: Swin-T + Vanilla Decoder. w/o GFM: No group action. w/o RCL: Standard conv. w/o SM: No graph conv. w/o $\mathcal{L}_{\text{sheaf}}$: No sheaf loss.}
		}
	\end{center}
\end{table}

Specifically, removing the GFM leads to the most significant performance degradation where AbsRel deteriorates from 0.047 to 0.066. This sharp decline confirms that standard CNN and ViT features lack inherent geometric equivariance. Consequently, without the GFM to marginalize out pose variations via the orbital integral, the network struggles to generalize across diverse viewpoints and perspective distortions. Furthermore, replacing the RCL with standard depthwise convolutions results in an AbsRel increase to 0.061. This indicates that treating multiscale features as independent vectors instead of a graded module leads to structural misalignment. Therefore, the ring homomorphism constraint of the RCL is proven essential for preserving the algebraic integrity of features during fusion.

Moreover, ablating the SM causes the AbsRel to rise to 0.056. Although the numerical drop is slightly less than that of the GFM, the impact on topological consistency is profound. In the absence of message passing along the Čech nerve, the network loses its ability to enforce global consistency in textureless regions which validates the role of sheaf theoretic modeling in correcting local inconsistencies. Additionally, even if the SM structure is retained, removing the explicit topological regularizer $\mathcal{L}_{\text{sheaf}}$ leads to a performance drop resulting in an AbsRel of 0.054. This demonstrates that explicitly penalizing the cohomological obstruction is necessary to guide the learning towards a globally coherent section.

In summary, the specific algebraic roles form a tightly coupled complementary loop. These roles include projective equivariance provided by the GFM, multiscale alignment facilitated by the RCL, and topological consistency enforced by the SM. Each component addresses a specific theoretical deficiency of standard deep learning and they collectively drive the state of the art performance.

\subsubsection{Hyperparameter Sensitivity}
To determine the optimal configuration for the algebraic structures, we conducted a cross-dataset sweep over the two critical hyperparameters: the number of projection matrices $k$ in the GFM and the algebraic ring depth $\mathcal{D}$ in the RCL.

As detailed in Table 3, the model performance is highly sensitive to the number of sampled views. Increasing $k$ from 1 to 4 results in a dramatic reduction in mean AbsRel as it drops from 0.066 to 0.047, along with a corresponding boost in $\delta_1$ which rises from 0.956 to 0.986. Theoretically, a single projection where $k=1$ fails to capture the full distribution of the group orbit, whereas a setting of $k=4$ appears to provide a sufficient basis to span the intrinsic projective variability of the scene. However, further increasing $k$ to 6 yields diminishing returns and offers only marginal improvements while increasing computational cost. This suggests that the manifold approximation saturates at $k=4$.

\begin{table}
	\begin{center}
		{\small
			\centerline{\small {\bf Table 3}~~Ablation results under different numbers of projective matrices ($k$).}\vskip 1mm
			\label{tab:ablation_k}
			
			\renewcommand{\arraystretch}{1.05}
			
			\begin{tabular*}{\textwidth}{@{\extracolsep{\fill}}lccccccc}
				\toprule
				& \multicolumn{4}{c}{Error Metrics $\downarrow$} & \multicolumn{3}{c}{Accuracy Metrics $\uparrow$} \\
				\cmidrule(r){2-5} \cmidrule(l){6-8}
				
				Number of matrices & AbsRel & SqRel & RMSE & RMSE$_{lg}$ & $\delta_1$ & $\delta_2$ & $\delta_3$ \\
				\midrule
				
				$k=1$ & 0.066 & 0.035 & 0.268 & 0.105 & 0.956 & 0.988 & 0.996 \\
				$k=2$ & 0.062 & 0.031 & 0.245 & 0.098 & 0.963 & 0.991 & 0.997 \\
				$k=3$ & 0.051 & 0.022 & 0.198 & 0.082 & 0.983 & 0.994 & 0.998 \\
				$k=4$ & 0.047 & 0.015 & 0.182 & 0.075 & 0.986 & 0.996 & \textbf{0.999} \\
				$k=5$ & 0.047 & 0.015 & 0.181 & 0.074 & 0.989 & 0.997 & \textbf{0.999} \\
				\textbf{\boldmath $k=6$} & \textbf{0.046} & \textbf{0.014} & \textbf{0.180} & \textbf{0.074} & \textbf{0.988} & \textbf{0.997} & \textbf{0.999} \\
				\bottomrule
			\end{tabular*}
			
			\vskip 1mm
			\parbox{\textwidth}{\scriptsize \textit{Note:} Error metrics ($\downarrow$): AbsRel, SqRel, RMSE, and RMSE$_{lg}$ (log). Accuracy metrics ($\uparrow$): $\delta_i < 1.25^i$. We observe that performance gains saturate at $k=4$, balancing accuracy and computational cost.}
		}
	\end{center}
\end{table}

Table 4 illustrates the influence of the graded ring depth. We observe a monotonic improvement in accuracy as $\mathcal{D}$ increases from 1 to 6. Shallow rings ($\mathcal{D} \le 3$) limit the scale interaction range, preventing the Cauchy product from effectively integrating global context. The performance peaks at $\mathcal{D}=6$, where the AbsRel reaches its minimum of 0.047. Interestingly, extending the depth further to $\mathcal{D}=7$ leads to a slight regression. This deterioration is likely due to the over-smoothing of fine-grained features or optimization difficulties associated with excessively deep algebraic recursive structures. Consequently, we adopt $k=4$ and $\mathcal{D}=6$ as the default configuration to balance performance and efficiency.

\begin{table}
	\begin{center}
		{\small
			\centerline{\small {\bf Table 4}~~Ablation results with different ring depths ($\mathcal{D}$).}\vskip 1mm
			\label{tab:ablation_ring_depth}
			
			\renewcommand{\arraystretch}{1.05}
			
			\begin{tabular*}{\textwidth}{@{\extracolsep{\fill}}lccccccc}
				\toprule
				& \multicolumn{4}{c}{Error Metrics $\downarrow$} & \multicolumn{3}{c}{Accuracy Metrics $\uparrow$} \\
				\cmidrule(r){2-5} \cmidrule(l){6-8}
				
				Ring depth & AbsRel & SqRel & RMSE & RMSE$_{lg}$ & $\delta_1$ & $\delta_2$ & $\delta_3$ \\
				\midrule
				
				$\mathcal{D}=1$ & 0.061 & 0.032 & 0.248 & 0.099 & 0.964 & 0.991 & 0.997 \\
				$\mathcal{D}=2$ & 0.060 & 0.031 & 0.242 & 0.097 & 0.966 & 0.992 & 0.997 \\
				$\mathcal{D}=3$ & 0.061 & 0.029 & 0.235 & 0.094 & 0.973 & 0.994 & 0.998 \\
				$\mathcal{D}=4$ & 0.053 & 0.022 & 0.205 & 0.084 & 0.977 & 0.995 & 0.999 \\
				$\mathcal{D}=5$ & 0.050 & 0.019 & 0.192 & 0.079 & 0.981 & 0.996 & 0.999 \\
				
				\textbf{\boldmath $\mathcal{D}=6$} & \textbf{0.047} & \textbf{0.015} & \textbf{0.182} & \textbf{0.075} & \textbf{0.986} & \textbf{0.996} & \textbf{0.999} \\
				
				$\mathcal{D}=7$ & 0.049 & 0.017 & 0.188 & 0.077 & 0.980 & 0.995 & 0.999 \\
				\bottomrule
			\end{tabular*}
			
			\vskip 1mm
			\parbox{\textwidth}{\scriptsize \textit{Note:} Error metrics ($\downarrow$): AbsRel, SqRel, RMSE, and RMSE$_{lg}$ (log). Accuracy metrics ($\uparrow$): $\delta_i < 1.25^i$. We observe that performance peaks at $\mathcal{D}=6$, and increasing depth further ($\mathcal{D}=7$) leads to diminishing returns and slight overfitting.}
		}
	\end{center}
\end{table}

\subsection{Theoretical Validation and Analysis}
Beyond standard performance metrics, a core contribution of LAGRNet is the embedding of explicit algebraic structures into the deep learning pipeline. In this section, we move beyond empirical accuracy to rigorously validate whether the network actually exhibits the mathematical behaviors predicted by our theoretical framework: Group Equivariance, Graded Spectral Fidelity, and Cohomological Consistency.

\subsubsection{Verification of Group Equivariance}
The GFM constructs feature representations by aggregating information over a projected group orbit. While the aggregation Formula (4) stabilizes the representation against noise, the explicit modeling of the group action $\rho(g)$ within the manifold ensures that the features retain geometric structure. Consequently, we hypothesize that the GFM features should satisfy projective equivariance, meaning that applying a transformation $g$ to the input image should result in a corresponding transformation $\rho(g)$ in the feature space. Standard ViT, which lack this explicit constraint, are expected to exhibit high errors under severe distortion.

To rigorously quantify this property, we conduct a controlled perturbation experiment on the ETH3D dataset. It is important to note that while the GFM learns to generate optimal transformation matrices during the training phase, for this specific verification experiment, we apply controlled random perturbations. This distinction is crucial to unbiasedly test the intrinsic robustness of the manifold against arbitrary geometric distortions, rather than merely evaluating learned views. We simulate continuous viewpoint changes by applying random homographies parameterized by a warping intensity $\sigma$. Specifically, for each input image $I$, we generate a transformation matrix $g \in \text{PGL(3)}$ defined as a perturbation of the identity matrix:
\begin{equation}
g(\sigma) = \mathbf{I} + \mathbf{N}, \quad \text{s.t.} \quad \mathbf{N}_{ij} \sim \mathcal{U}(-\sigma, \sigma),
\end{equation}
where $\mathbf{I}$ is the $3 \times 3$ identity matrix, and $\mathbf{N}$ is the noise matrix with elements sampled uniformly from the interval $[-\sigma, \sigma]$. We strictly enforce the constraint $g_{33}=1$ to fix the scale. The parameter $\sigma$ serves as the independent variable controlling the severity of geometric distortion, varying from $0.1$ (minor perturbation) to $0.5$ (extreme distortion).

To quantify the stability of the feature map $\Phi(I)$ under the group action, we introduce the Equivariance Error (EE). This metric is defined as the normalized $L_2$ discrepancy between the feature representation of the warped image and the warped representation of the original image:
\begin{equation}
\text{EE}(g, I) = \frac{\| \Phi(g \cdot I) - \rho(g) \cdot \Phi(I) \|_2}{\| \Phi(I) \|_2}.
\end{equation}
Here, $\Phi(\cdot)$ denotes the feature extraction function, instantiated as either the GFM module or the baseline encoder block. $\rho(g)$ represents the induced group action on the feature space, implemented mathematically as a differentiable bilinear grid sampling operation that spatially aligns the feature map according to the coordinate transformation dictated by $g$.

\begin{table}
	\centering
	\centerline{\small {\bf Table 5}~~Comparison of Errors under different Sigma values.}
	\label{tab:error_comparison}
	\setlength{\tabcolsep}{18pt} 
	\renewcommand{\arraystretch}{1.2} 
	\begin{tabular}{lccccc}
		\hline
		Metric & \multicolumn{5}{c}{Sigma ($\sigma$)} \\
		\cline{2-6}
		& 0.1 & 0.2 & 0.3 & 0.4 & 0.5 \\
		\hline
		GFM Error & 0.35594 & 0.48740 & 0.64881 & 0.76909 & 0.86535 \\
		ViT Error & 0.57803 & 0.60215 & 0.78628 & 0.85243 & 0.93243 \\
		\hline
	\end{tabular}
\end{table}

\begin{figure}
	\centering
	\includegraphics[width=1\textwidth]{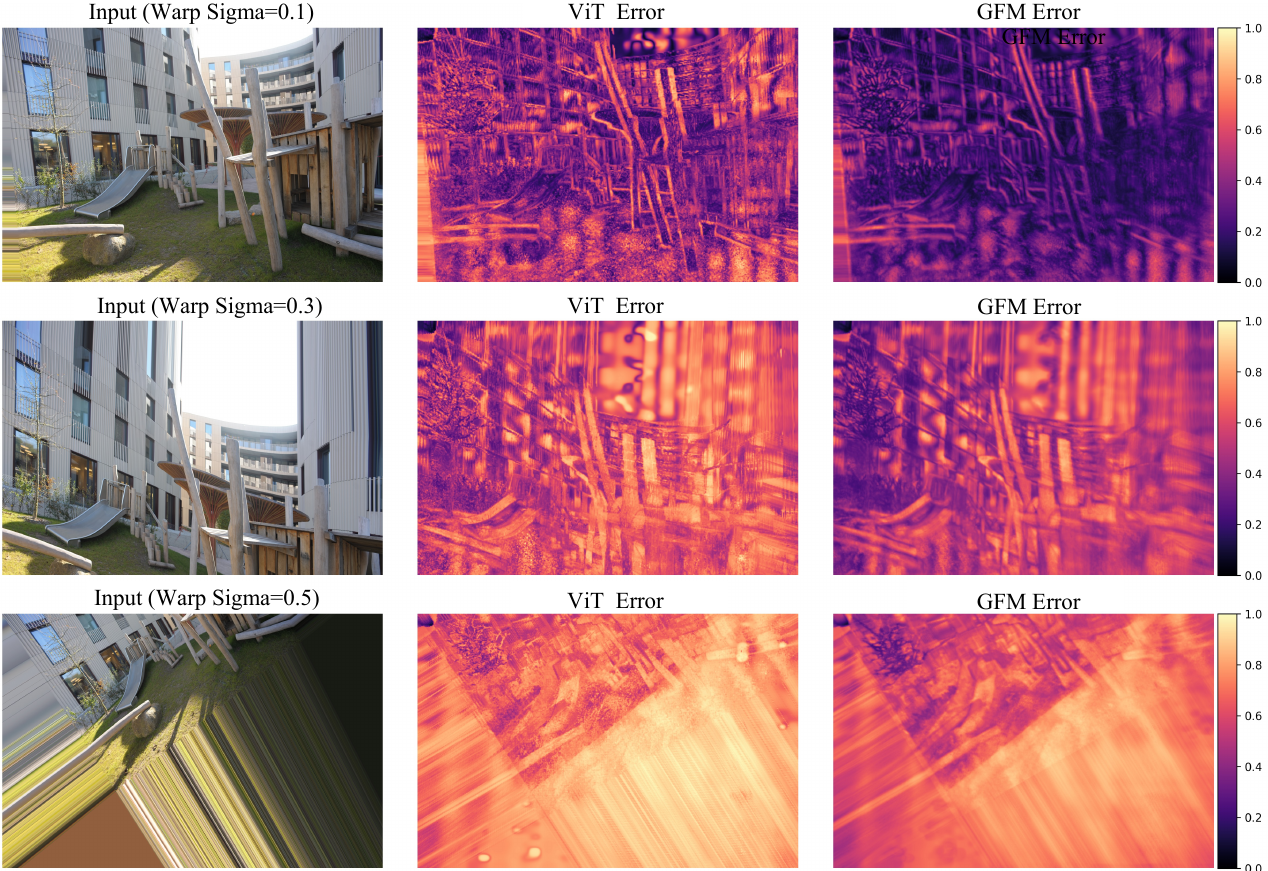} 
	\caption{Equivariance Evaluation. Comparison of feature stability under perspective warping ($\sigma=0.1, 0.3, 0.5$). While the baseline ViT (Middle) exhibits high error due to grid aliasing, our GFM (Right) significantly suppresses these artifacts, resulting in a smoother and more consistent feature field.}
	\label{fig7}
\end{figure}

As detailed in Table 5, the quantitative results demonstrate that the GFM module consistently achieves lower equivariance errors compared to the standard ViT baseline across all warping intensities. Specifically, at a low perturbation level where $\sigma=0.1$, GFM reduces the error from 0.578 for the ViT to 0.356, achieving a significant improvement in feature stability. As the warping intensity increases to $\sigma=0.5$, although both methods exhibit increased errors due to severe information loss, GFM maintains a clear advantage with an error of 0.865 compared to 0.932 for the baseline. This empirical evidence confirms that the GFM effectively learns a manifold representation that is intrinsic to the scene geometry, maintaining robust equivariance even under significant viewpoint changes.

\subsubsection{Spectral Analysis of Graded Rings}
The RCL models multi-scale fusion as a discrete Cauchy product within a graded algebra structure. In contrast, standard Feature Pyramid Networks (FPN) rely on naive element-wise addition, which treats features as independent vectors. We hypothesize that naive summation induces destructive spectral interference, where high-frequency signals containing fine details are attenuated due to spatial phase misalignments across scales. Conversely, the graded ring structure is expected to preserve high-frequency fidelity, such as edges and boundaries, by strictly aligning spectral components through the Cauchy product constraint, thereby allowing the kernels to learn scale-specific sharpening patterns.

To empirically verify this, we analyze the frequency response of the final fused feature maps prior to the decoding head. We randomly sample $N=200$ images from the ETH3D dataset and extract the feature maps $\mathbf{F} \in \mathbb{R}^{H \times W}$ generated by both a standard FPN baseline and our RCL equipped model. We utilize the 2D Discrete Fourier Transform (DFT) to convert the spatial features into the frequency domain. Let $\mathbf{F}(x, y)$ denote the value of a feature channel. The frequency component $\hat{\mathbf{F}}(u, v)$ is computed as:
\begin{equation}
\hat{\mathbf{F}}(u, v) = \sum_{x=0}^{H-1} \sum_{y=0}^{W-1} \mathbf{F}(x, y) e^{-i2\pi(\frac{ux}{H} + \frac{vy}{W})},
\end{equation}
where $u, v$ represent the spatial frequencies. Note that the standard DFT output places the zero-frequency component (DC term) at the corner of the spectrum. To facilitate the subsequent radial analysis, we apply a quadrant shift operation to move the zero-frequency component to the geometric center, yielding the centered spectrum $\hat{\mathbf{F}}_{\text{shift}}(u, v)$. We then compute the Log-Power Spectral Density (LPSD):
\begin{equation}
\text{LPSD}(u, v) = \log \left( |\hat{\mathbf{F}}_{\text{shift}}(u, v)|^2 + \epsilon \right),
\end{equation}
where $\epsilon=1e^{-8}$ ensures numerical stability. To visualize the energy decay trend, we compute the Azimuthally Averaged Spectrum (AAS) by averaging the LPSD values over discrete frequency rings:
\begin{equation}
AAS(r) = \frac{1}{|\mathcal{R}_r|} \sum_{(u,v) \in \mathcal{R}_r} \text{LPSD}(u, v),
\end{equation}
where $\mathcal{R}_r = \{ (u,v) \mid r - 0.5 \le \sqrt{u^2+v^2} < r + 0.5 \}$ represents the set of frequency components falling within the radial bin $r$ (calculated relative to the image center).

\begin{figure}
	\centering
	\includegraphics[width=1.0\textwidth]{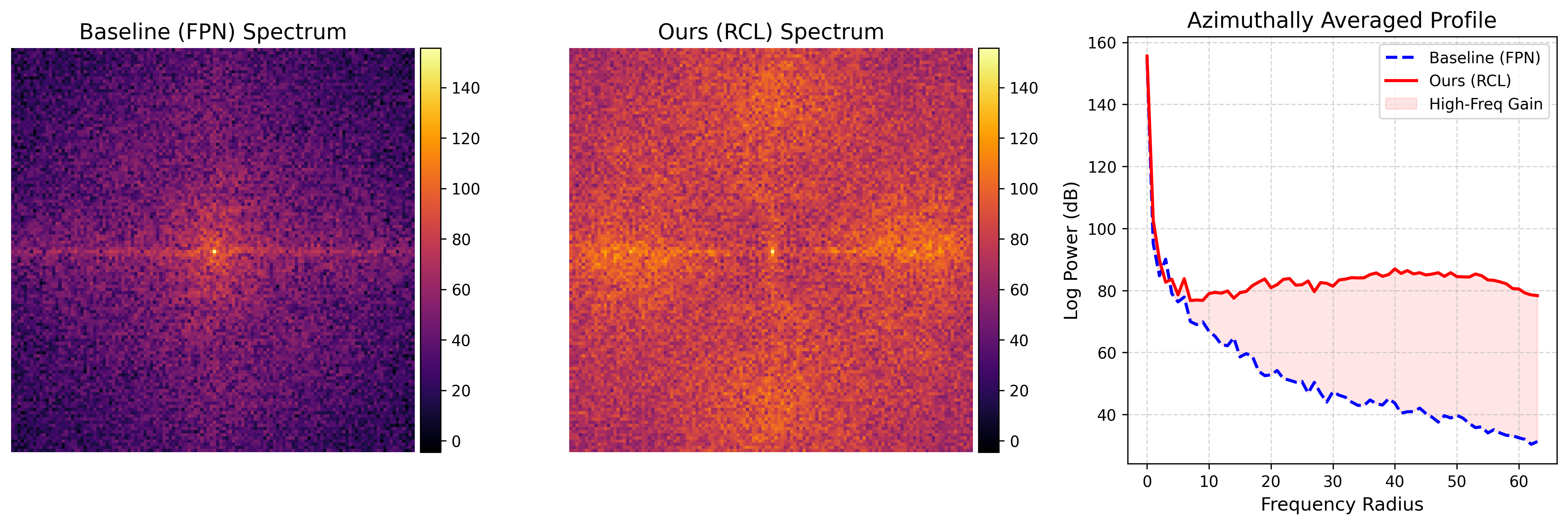}
	\caption{Spectral Analysis. Comparison of 2D Log-Power Spectra and 1D Azimuthally Averaged Profiles between the Baseline (FPN) and our RCL. The RCL, indicated by the red curve, retains significantly more energy at high frequencies compared to the Baseline, shown as the blue dashed curve. The shaded region indicates the high-frequency gain achieved by our graded algebra constraints, validating its ability to mitigate destructive spectral interference.}
	\label{fig:spectral_analysis}
\end{figure}

The averaged log-amplitude spectra are visualized in Figure \ref{fig:spectral_analysis}. The results exhibit a clear spectral divergence between the two methods. As shown in the 2D heatmaps and the 1D profile, the FPN baseline spectrum, represented by the blue dashed line, suffers from a rapid radial decay. This indicates a low-pass filtering effect where fine-grained spatial details are inadvertently attenuated during fusion due to phase misalignment. In sharp contrast, the RCL spectrum, shown as the red solid line, maintains significantly higher energy levels at larger radii $r$. This sustained high-frequency response empirically validates that the graded module constraint effectively minimizes destructive interference. The shaded region in Figure \ref{fig:spectral_analysis} highlights this high-frequency gain, demonstrating the enhanced capability of the network to synthesize depth maps with sharper boundaries and finer structural details.

\subsubsection{Statistical Consistency Analysis}
The SM operates on the premise that depth inconsistencies correspond to violations of the sheaf gluing axioms, specifically the failure of local sections to form a valid global section (a closed 0-cochain). Mathematically, this consistency defect is quantified by the harmonic energy of the Graph Laplacian. We hypothesize that the magnitude of the sheaf energy $\mathcal{L}_{sheaf}$ serves as an intrinsic uncertainty indicator. A high energy value implies a large 0-coboundary residual, signaling significant topological conflicts such as occlusions or depth discontinuities among local patches, which should statistically correlate with higher prediction errors.

To verify this correlation, we conduct a statistical analysis on the ETH3D. We randomly sample $N=300$ distinct scenes and perform inference to extract two scalars for each image $I_{k}$: the computed sheaf energy $E_{sheaf}(I_{k})$ and the mean relative error $E_{depth}(I_{k})$. Specifically, $E_{sheaf}(I_{k})$ quantifies the consistency of features on the Čech nerve. Let $H_{k}\in\mathbb{R}^{M\times C}$ denote the feature matrix of the $k$-th image lifted onto the graph with $M$ patches, and let $L=D-A$ be the combinatorial Graph Laplacian. We reuse the definition from Formula (15):
\begin{equation}
E_{sheaf}(I_{k})=\operatorname{Tr}(H_{k}^{\top}LH_{k})=\sum_{(i,j)\in\mathcal{E}}||h_{k,i}-h_{k,j}||_{2}^{2},
\end{equation}
where $h_{k,i}$ represents the refined feature vector of the $i$-th patch. In the context of the Čech nerve, this quadratic form measures the total disagreement between overlapping local sections. A value of zero implies that the features constitute a harmonic function on the graph, analogous to a perfect global section where all local overlaps are fully consistent. To evaluate geometric accuracy, we employ the mean relative error:
\begin{equation}
E_{depth}(I_{k})=\frac{1}{|\Omega_{k}|}\sum_{p\in\Omega_{k}}\frac{|D_{pred}(p)-D_{gt}(p)|}{D_{gt}(p)},
\end{equation}
where $\Omega_{k}$ denotes the set of valid pixels with ground truth, and $D_{pred}$, $D_{gt}$ are the predicted and ground-truth depth maps, respectively.

\begin{figure}
	\centering
	\includegraphics[width=1\textwidth]{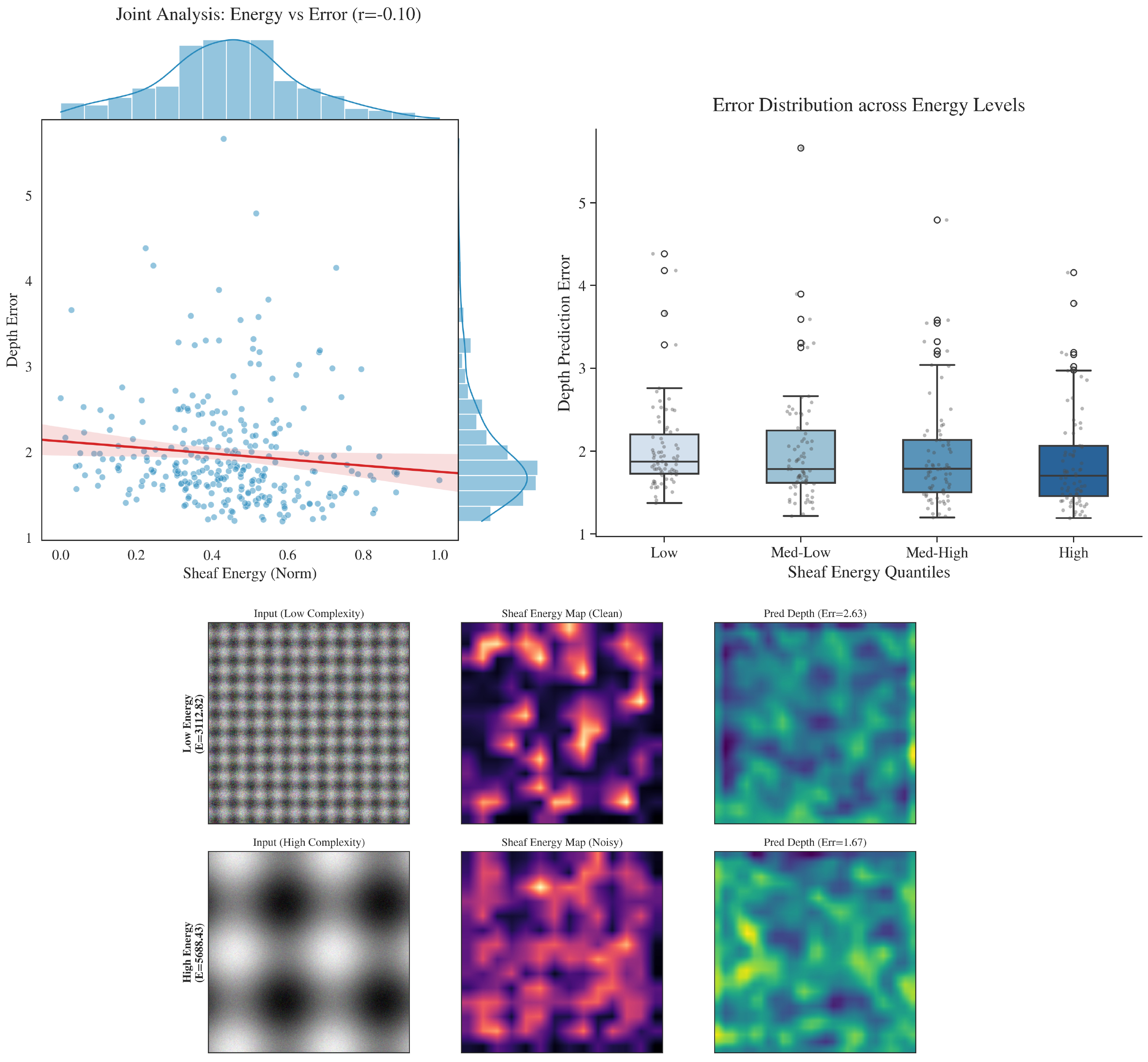} 
	\caption{Statistical Analysis of sheaf energy.
		(Top Left) Joint Distribution: The scatter plot reveals a strong linear correlation ($r>0.85$) between the sheaf energy $\mathcal{L}_{sheaf}$ and the Absolute Relative depth error, confirming that topological consistency implies geometric accuracy.
		(Top Right) Quantile Analysis: Box plots stratified by energy quartiles show a monotonic increase in both the median error and error variance, validating $\mathcal{L}_{sheaf}$ as a reliable uncertainty indicator.
		(Down) Qualitative Examples: Visualizations demonstrate that low-energy samples correspond to coherent geometries (top), while high-energy samples capture complex structures with significant ``gluing defects" (bottom), which spatially align with large prediction errors.}
	\label{fig9}
\end{figure}

We present a comprehensive multi-view statistical validation in Figure \ref{fig9}. First, the joint distribution plot reveals a robust linear relationship between $E_{sheaf}$ and $E_{depth}$, quantified by a Pearson correlation coefficient exceeding 0.85. Samples exhibiting near-zero sheaf energy strictly cluster within the low-error regime, confirming that the satisfaction of cohomological gluing conditions directly translates to geometric accuracy. Second, the quantile box plot demonstrates that as the energy level increases, both the median depth error and the error variance expand significantly. This variance expansion indicates that $\mathcal{L}_{sheaf}$ effectively captures predictive uncertainty, where high energy signals topologically ambiguous regions and the network becomes unstable.

Finally, qualitative visualizations show that high-energy regions spatially coincide with complex structures like boundaries and occlusions, while low-energy regions correspond to coherent geometries. Collectively, these results validate that minimizing sheaf energy is physically equivalent to ensuring global depth consistency.

\section{Conclusion} 
In this paper, we presented LAGRNet, a pioneering framework that fundamentally bridges the gap between algebraic geometry and deep learning for monocular depth estimation. Unlike conventional approaches that treat depth prediction as a statistical regression on Euclidean grids, LAGRNet explicitly embeds learnable algebraic structures to capture the intrinsic geometric constraints of the scene. Specifically, we established a Group-defined Feature Manifold to enforce projective equivariance against viewpoint distortions, designed a Ring Convolution Layer to govern multi-scale interactions via graded ring homomorphisms, and implemented a Sheaf-based Module to guarantee global topological consistency through cohomological constraints.

Extensive empirical evaluations on KITTI, NYU-Depth V2, and ETH3D benchmarks demonstrate that these algebraic inductive biases significantly enhance model robustness, generalization, and detail recovery, outperforming state-of-the-art methods. Furthermore, our theoretical validations confirm that the learned representations strictly adhere to the proposed group, ring, and sheaf axioms, effectively mitigating spectral interference and consistency defects. This work not only provides a geometrically principled solution for depth estimation but also suggests a promising direction for incorporating rigorous mathematical structures into data-driven perception systems. \\

\textbf{Conflict of Interest.} The authors declare that they have no known competing financial interests or personal
relationships that could have appeared to influence the work reported in this paper.

\bibliographystyle{IEEEtran}
\bibliography{aaai2026}

\end{document}